\documentclass{article}

     \usepackage{arxiv}

\usepackage[utf8]{inputenc} 
\usepackage[T1]{fontenc}    
\usepackage{hyperref}       
\usepackage{url}            
\usepackage{booktabs}       
\usepackage{amsfonts}       
\usepackage{nicefrac}       
\usepackage{microtype}      

\usepackage{subcaption}
\usepackage{flushend}
\usepackage{tikz}
\usepackage{color}
\usepackage{amsmath, amsfonts, amssymb, mathrsfs}
\usepackage[amsmath,thmmarks]{ntheorem}

\usepackage{pgfplots}
\pgfplotsset{compat=newest}
\pgfplotsset{plot coordinates/math parser=false}

\usepackage{nth}

\usepackage{algorithm}
\usepackage{algorithmic}
\usepackage{bbm}

\usepackage[draft]{todonotes} 

\def\X{\mathcal{X}}
\def\Z{\mathcal{Z}}
\def\V{\mathcal{V}}
\def\Prob{\mathbb{P}}

\def\real{ \mathbb{R} }
\def\nat{ \mathbb{N} }
\def\Erw{ \mathbb{E} }	

\def\T{ \mathrm{T} }

\newcommand{\norm}[1]{\| #1 \|}
\newcommand{\inn}[2]{\langle #1,#2 \rangle}
\newcommand{\lrbrace}[1]{\left\{ #1 \right\}}

\newcommand{\normi}[1]{{\left\vert\kern-0.25ex\left\vert\kern-0.25ex\left\vert #1 
		\right\vert\kern-0.25ex\right\vert\kern-0.25ex\right\vert}}
\newcommand{\inni}[2]{{\langle\kern-0.25ex\langle #1,#2
		\rangle\kern-0.25ex\rangle}}

\DeclareMathOperator{\VI}{\text{VI}}

\DeclareMathOperator{\SOL}{\text{SOL}}

\DeclareMathOperator{\ANCCVC}{\text{ANCCVC}}

\DeclareMathOperator*{\argmax}{arg\,max}

\theoremstyle{plain}
\theoremheaderfont{\itshape\bfseries}
\theorembodyfont{\normalfont\itshape}
\theoremseparator{:}
\theoremsymbol{}
\newtheorem{theorem}{Theorem}
\newtheorem{lemma}[theorem]{Lemma}

\newtheorem{proposition}[theorem]{Proposition}

\newtheorem{definition}{Definition}

\theoremstyle{nonumberplain}

\theoremstyle{plain}
\theoremheaderfont{\itshape\bfseries}
\theorembodyfont{\normalfont}
\theoremseparator{:}
\theoremsymbol{}
\newtheorem{remark}{Remark}

\theoremstyle{plain}
\theoremheaderfont{\itshape\bfseries}
\theorembodyfont{\normalfont}
\theoremseparator{:}

\newtheorem{assum}{Assumption}

\theoremstyle{nonumberplain}
\theoremsymbol{\rule{1ex}{1ex}}

\newtheorem{proof}{Proof}
\newlength\fheight
\newlength\fwidth
\title{Pricing Mechanism for Resource Sustainability in Competitive Online Learning Multi-Agent Systems}

\author{
	Ezra Tampubolon\\
	Department of Electrical and Computer Engineering\\
	Technical University Munich\\
	\texttt{ezra.tampubolon@tum.de} \\
	\And
	Holger Boche\\
	Department of Electrical and Computer Engineering\\
	Technical University Munich\\
	\texttt{boche@tum.de} \\
}

\begin{document}

\maketitle

\begin{abstract}
	In this paper, we consider the problem of resource congestion control for competing online learning agents.
	On the basis of non-cooperative game as the model for the interaction between the agents, and the noisy online mirror ascent as the model for rational behaviour of the agents, we propose a novel pricing mechanism which gives the agents incentives for sustainable use of the resources. Our mechanism is distributed and resource-centric, in the sense that it is done by the resources themselves and not by a centralized instance, and that it is based rather on the congestion state of the resources than the preferences of the agents. In case that the noise is persistent, and for several choices of the intrinsic parameter of the agents, such as their learning rate, and of the mechanism parameters, such as the learning rate of -, the progressivity of the price-setters, and the extrinsic price sensitivity of the agents, we show that the accumulative violation of the resource constraints of the resulted iterates is sub-linear w.r.t. the time horizon. Moreover, we provide numerical simulations to support our theoretical findings. 
\end{abstract}

\section{Introduction}
\textbf{Online Mirror Descent - Rationality in the Face of Unknown:}
\emph{Online learning} has become an important concept for real-time decision making in an unknown environment, and has led to several efficient methods in widespread applications (e.g. see \cite{Li2018,Chen2018}). Its process can be formulated as follows: at each time $t$, a learner selects an action $x_{t}$ from a set $\X$  and suffers the \emph{loss} $f_{t}(x_{t})$ specified by a function $f_{t}:\X\rightarrow\real$ apriori unknown to her. By means of additional information about the environment state, such as the gradient of $f_{t}$, the learner chooses her next action with the aim of minimizing her loss. The quality of a learning policy is measured by its \emph{regret} $\text{Reg}_{t}:=\max_{x\in\mathcal{X}}\sum_{\tau=1}^{t}[f_{\tau}(x)-f_{\tau}(X_{\tau})]$ and a satisfactory one is characterized by the \emph{no-regret} property (see e.g. \cite{Shalev-Shwartz2012,Bub2012,Belmaga2018}), i.e. sub-linear decay of its regret with time. The canonical and widely used class of no-regret policy in the presence of first-order oracle is the so-called \emph{online mirror descent (OMD)} \cite{Shalev2007}, known also in other literature as dual averaging \cite{Nesterov2009}. The iterate of OMD consists of averaging process of the obtained first-order information giving the \emph{score vector}, and subsequent \emph{mirror step} realizing the score in the action space $\X$. 

\textbf{Game Theory - Competition in an Online Environment:}
As has already been recognized in \cite{Zinkevich2003}, the concept of online learning can serve as a paradigm to describe the decision making process of rational agents in a \emph{non-cooperative game (NG)}. NG is a popular model, not only for economics - and social perspectives, but also for vast number of real-world technical applications (see also \cite{Scutari2012}), especially in those where cooperation between system participants is hard to establish, such as smart grid
\cite{Mohsenian2010,Saad2012,Deng2014,Li2016,Ma2016},
  networked system \cite{Barrera2015}, or in general large-scale complex systems (e.g. those which emerge within the framework of internet of things), where cooperation between system participants is hard to establish. The typical setting of NG is as follows: the reward $u(x_{t}^{(i)},x_{t}^{(-i)})$ of each agent $i$ obtained in a time slot $t$ depends on both, her action $x_{t}^{(i)}$ and the action others $x_{t}^{(-i)}$. To model the non-cooperativeness aspect, the latter is assumed to be not visible apriori for agent $i$. Thus her reward can be expressed as $u^{(i)}_{t}(x^{(i)}_{t})$ where $u^{(i)}_{t}$ is an apriori unknown function. This justifies the assumption of competitive rational agents as online learners.

\textbf{Resource Constraints:}
In widespread practical applications, the action of the agents is additionally related to the utilization of certain limited resources. For example: in the network applications, the user's (agents) choice of data transfer paths (strategy) increases the congestion of certain links and routers (resources) with limited capacity; In electric mobility, the vehicles' (agents) charge policy (strategy) increases the load of a grid, having limited electrical power (resource), at certain times \cite{Ma2013}.
An important issue which has to be dealt by a system designer and - manager is the danger of resource overload due to agents' egoistic behaviour, because the state of overutilization of resources can caused immense degradation of the overall system performance (see e.g. the problem of congestion and congestion collapse in networked system \cite{Abbas2016}) and negative environmental issues (e.g. caused by high CO2 emissions of electrical energy driven resources). Another example of events justifying the importance of sustainability aspect in a system of egoistic optimizing agents is the flash crash in US financial markets due to fully automated computerized trading (see e.g. \cite{Borch2016}).       

\textbf{Problem Description:}
This work addresses the problem of how to avoid or at least alleviate resource congestion in a system consisting of selfish online learning agents in competitive environment. A challenge associated with this issue is to design a congestion control method which demands as few information about agents' characteristics as possible. The reason is that the methods contrary to the latter requirement would need, in case the number of agents is large, exceptionally high computational power for the processing of the obtained information. Moreover, such methods would be inflexible for possible exit of - and entrance of new agents and therefore unsuitable for modern systems such as IoT.   

\textbf{Our Contribution:} Align with the postulated requirement, we provide a novel price-based congestion control method aiming to give incentives to egoistic OMD-applying agents for sustainable use of resources. The pricing is based on the congestion state of the resources and is done by the resources themselves rather than by a centralized instance. Furthermore, we provide sub-linear bound for the cumulative violation of the resource constraints and decaying bound for the violation of the constraints made by ergodic average of the population action. We are not aware of a comparable control methods in the learning literature, since it either focuses on the behaviour of a single learner or hides further possible influence of learners' decisions to certain scarce resources. Also, we are not aware of similar non-asymptotic focus on the resource aware behaviour in the game theory literature.

\textbf{Relation to prior works}
 One of the closest works to ours is \cite{Mertikopoulos2018}. There, an analysis of OMD with noisy feedback for NG with continuous action set is given. Several interesting extensions of this work have been made: The work \cite{DuvocellaMertikopoulos2018} extends the analysis for cases where the utility functions of the agents are time variant, and \cite{ZhouMertikopoulos2018} for cases where the feedback received by the agents is delayed. In contrast to \cite{Mertikopoulos2018}, we consider NG which underly in addition resource constraints. Our focus is not on the stability of the population dynamic, but but rather on the resource constraints awareness.
Another works closely related to ours are works which focus on algorithmically finding generalized Nash equilibrium. The most recent one is \cite{Paccagnan2017}. Based on a fixed-point method for finding the solution of a variational inequality (see e.g. Chapter 12 in \cite{Facchinei1}), a Euclidean projection-based semi-decentralized algorithm converging to a state fulfilling coupled resource constraints are fulfilled. In contrast to our result, their results are purely of asymptotic nature, i.e. there is no guarantee for the population behaviour in the finite time.
 
 
\section{Model Description and Preliminaries}
\textbf{Basic notations:}
For $a\in\nat$, we denote $[a]:=\lrbrace{1,\ldots,a}$ and $[a]_{0}=[a]\cup\lrbrace{0}$. For a convex subset $A\subseteq\real^{D}$, $\text{relint}(A)$ denotes the relative interior of $A$. The projection onto a closed convex subset $A$ of $\real^{D}$ is denoted by $\Pi_{A}$. The dual norm of a norm $\norm{\cdot}$ on $\real^{D}$ is denoted by $\norm{\cdot}_{*}$. 
$F:\real^{D}\rightarrow\real^{D}$ is said to be monotone on $\Z$ if $\inn{x_{1}-x_{2}}{F(x_{1})-F(x_{2})}\leq 0$, for all $x_{1},x_{2}\in \mathcal{Z}$. If in the latter strict inequality hold for $x_{1}\neq x_{2}$, then $F$ is said to be strictly monotone. $F$ is said to be $c$-strongly monotone on $\Z$ if $\inn{x_{1}-x_{2}}{F(x_{1})-F(x_{2})}\leq -c \norm{x_{1}-x_{2}}^{2}$, for all $x_{1},x_{2}\in \mathcal{Z}$. 
\subsection{Continuous Game with Coupled Resource Constraints}
\textbf{Continuous Non-Cooperative Game:} We consider throughout this work a finite set $[N]$ of agents playing a (repeated) \emph{non-cooperative game (NG)} $\Gamma$. During the NG, every agent $i\in[N]$ chooses and applies an action/strategy $x^{(i)}$ from a non-empty compact convex subset $\X_{i}$ of a finite-dimensional normed space $(\real^{D_{i}},\norm{\cdot}_{i})$. This process results in joint action/strategy-profile $x=(x^{(1)},\ldots,x^{(N)})\in\X:=\prod_{i=1}^{N}\X_{i}\in\real^{D}$, where $D:=\sum_{i=1}^{N}D_{i}$. In order to highlight the action of player $i$, we write $x=(x^{(i)},x^{(-i)})$ where $x^{(-i)}=(x^{(j)})_{j\neq i}\in \mathcal{X}_{-i}:=\prod_{j\neq i}\X_{j}$. Suppose that the population action at time $t$ is $x_{t}\in\mathcal{X}$. The payoff/reward agent $i$ received after $x_{t}$ is given by $u_{i}(x_{t}^{(i)},x_{t}^{(-i)})$, where $u_{i}:\X\rightarrow\real$ is a coordinate-wise concave and continuously differentiable function. 



\textbf{Coupled Resource Constraints:} 
For a certain number $R>0$ of resources, we model the relation between agents' action and resource utilization circumstance is modeled by an affine function $\phi=(\phi^{1},\ldots,\phi^{R}):\real^{D}\rightarrow\real^{R}$, $x\mapsto \mathbf{A}x-b$, where $\mathbf{A}:=[\mathbf{A}_{1},\ldots,\mathbf{A}_{N}]\in\real^{R\times D}$ with $\mathbf{A}_{i}\in\real^{M\times D_{i}}$, $\forall i\in [N]$, specifies the resource load caused by a population's action, and where $b\in\real^{R}$ describes the capacity of resources. Correspondingly, $\phi^{r}(x)$ gives the overload/congestion state of the resource $r\in [R]$ caused by the population action $x$. Since from operational - and sustainability point of view overload has to be kept low and even avoided, it is desired that the population strategy is contained in $\mathcal{Q}:=\mathcal{C}\cap\mathcal{X}$, where $\mathcal{C}:=\{\phi(x)\leq 0\}$. In order that this goal is feasible, we assume that $\mathcal{C}$ is non-empty. A sufficient condition leading to the latter circumstance is the following:
\begin{assum}[Slater's condition]
	\label{Eq:aioaoosjjshhdjjddss}
	There exists $x_{*}\in\text{relint}{(\X)}$ s.t. $\phi(x_{*})< 0$.
\end{assum}
This sort of constraint is a subclass of the so called \textit{coupled constraint}, where the compliance depends on the strategy not only of a single agent but also of the whole population.  
%


 \subsection{Basic Agents' Behaviour in the Repeated NG}
\paragraph{First-Order Feedback with Martingale Noise} Let be $i\in[N]$. Suppose that in the time slot $t$ the agents $[N]\setminus \lrbrace{i}$ has applied the action $x_{t}^{(-i)}\in\mathcal{X}_{-i}$. To improve her payoff, agent $i$ may use (if available) her individual utility gradient: $$v_{i}((\cdot),x^{(-i)}_{t}):\mathcal{X}_{i}\rightarrow\real^{D_{i}},\quad x_{t}^{(i)}\mapsto\nabla_{x_{t}^{(i)}}u_{i}(x_{t}^{(i)},x_{t}^{(-i)}),$$
and go along the direction of steepest ascent of her utility function. However, perfect first-order feedback is in general hard to obtain, especially without explicit knowledge of the utility function. Thus, we include in our model the possibility that agent $i$ has only access to noisy first-order oracle rather than a perfect one. Specifically, we assume that at each time $t$ and for a given action $X_{t}\in\mathcal{X}$, agent $i$ can query an estimate of $\hat{v}^{(i)}_{t}$ of $v_{i}(X_{t})$ satisfying $\Erw[\norm{\hat{v}^{(i)}_{t}}_{*}]<\infty$ and $\Erw[\hat{v}^{(i)}_{t}|\mathcal{F}_{t}]=v_{i}(X_{t})$, 
 where $\mathcal{F}_{t}$ is an element of a filtration $\mathbb{F}:=(\mathcal{F}_{t})_{t\in\nat_{0}}$ on a probability space $(\Omega,\Sigma,\Prob)$, which we assume throughout this work to be given. The canonical and commonly-used filtration in the literature is the filtration of the history of the considered iterates. Equivalently, we can model the stochastic gradient by
 \begin{equation}
 \label{Eq:NoiseGrad}
  \hat{v}^{(i)}_{t}=v_{i}(X_{t})+\xi^{(i)}_{t+1},
 \end{equation}
   where $(\xi^{(i)}_{t})_{t\in\nat}$ be a $\real^{D_{i}}$-valued \emph{$\mathbb{F}$-martingale difference sequence}, i.e. it is \emph{$\mathbb{F}$-adapted}, in the sense that  $\xi_{t}$ is $\mathcal{F}_{t}$-measureable for all $t\in\nat$, and that its members are \emph{conditionally mean zero}, in the sense that $\Erw[\xi_{t}|\mathcal{F}_{t-1}]=0$, for all $t\in\nat$.
 If we work with the whole population, we use the following notations:
 \begin{equation*}
 v:\mathcal{X}\rightarrow\real^{D},~x\mapsto (v_{i}(x^{(i)},x^{(-i)}))_{i\in [N]},\quad\text{and}\quad\xi_{t}:=(\xi^{(i)}_{t})_{i\in [N]}
 \end{equation*}

\paragraph{Mirror Map - Realizing Action in the constraint set}
The following gives the map which allows the agents to project the iterate based on first-order informations to their individual constraint sets:
%
\begin{definition}[Regularizer/penalty function and Mirror Map]
	Let $\Z$ be a compact convex subset of a normed space $(\real^{M},\norm{\cdot})$, and $K>0$. We say $\psi:\Z\rightarrow\mathbb{R}$ is a $K$-strongly convex \textit{regularizer} (or \textit{penalty function}) on $\Z$, if $\psi$ is continuous and $K$-strongly convex on $\Z$. 
	The mirror map $\Phi:(\real^{M},\norm{\cdot}_{*})\rightarrow\Z$ induced by $\psi$ is defined by:
	$\Phi(y):=\argmax_{x\in\Z}\left\{\left\langle y,x\right\rangle-\psi(x)\right\}$
\end{definition}

Clearly, the mirror map is a generalization of the usual Euclidean projection.
An interesting example of mirror maps is the so-called logit choice $\Phi(y)=\exp(y)/\sum_{l=1}^{D}\exp(y_{l})$ which is generated by the $1$-strongly convex regularizer $\psi(x)=\sum_{k=1}^{D}x_{k}\log x_{k}$ on the probability simplex $\Delta\subset(\real^{D},\norm{\cdot}_{1})$. 
Throughout this work, we assume that each agent $i\in [N]$ possess a $K_{i}$-strongly convex regularizer $\psi_{i}$ which induces the mirror map $\Phi_{i}$, and the Fenchel coupling $F_{i}$. In order to emphasize the action of the whole population, we use the operator $\Phi:\real^{D+R}\rightarrow\X$, $y\mapsto (\Phi_{1}(y^{(1)}),\ldots,\Phi_{N}(y^{(N)}))$ and the total Fenchel coupling $F:\X\times\real^{D}$, $(x,y)\rightarrow\sum_{i=1}^{N}F_{i}(x_{i},y_{i})$.

\paragraph{Online Mirror Descent}
The foundation of the investigations made in this work is given by the following decision-making model of the agent $i\in [N]$: 
\begin{equation}
\label{Eq:aoaaosjsdhdjdhhddss2}
X^{(i)}_{t+1}=\Phi_{i}(Y_{t+1}^{(i)}),~Y^{(i)}_{t+1}=Y^{(i)}_{t}+\gamma_{t}\hat{v}^{(i)}_{t},\quad \gamma_{t}>0
\end{equation}
\eqref{Eq:aoaaosjsdhdjdhhddss2} is a canonical extension (see \cite{Mertikopoulos2018}) of the online mirror descent algorithm \cite{Shalev-Shwartz2012} to multi-agent competitive system. By writing the recursion in \eqref{Eq:aoaaosjsdhdjdhhddss2} explicitly, we obtain $Y^{(i)}_{t+1}=Y_{0}^{(i)}+\sum_{\tau=0}^{t}\gamma_{\tau}\hat{v}^{(i)}_{\tau}$, and consequently see that the decision of agent $i$ results in this context from action-scoring process by averaging the historical direction of the steepest ascent of her utility, whereby the aspect of "action-scoring" is best seen in the case $\mathcal{X}_{i}$ is the probability simplex and $\Phi_{i}$ is the logit choice. 


\subsection{Nash Equilibrium and Variational Inequality}
\label{SubSec:aoosjsjdhhdjdhhsggshhsgss}
One of the central concept in game theory is the so-called Nash equilibrium (NE) which denotes a feasible population strategy profile at which no agent can improve his reward by unilaterally deviating from his strategy. Besides the aspect that a NE gives a solution and thus a prediction of the decisions of rational agents in a competitive one-shot environment, it has the potential to be an equilibrium in the repeated NG setting with first-order feedback. Key to the understanding of this statement is the following equivalent concept: 
\begin{definition}[Variational Inequality (VI)]
	Let $\Z$ be a subset of a finite dimensional normed space $(\real^{M},\norm{\cdot})$, and suppose that $F:\mathcal{Z}\rightarrow (\real^{M},\norm{\cdot}_{*})$. A point $\overline{x}\in\Z$ is a solution of the variational inequality $\text{VI}(\mathcal{Z},F)$, if $\left\langle x-\overline{x} ,F(\overline{x})\right\rangle\leq 0$, for all $x\in\mathcal{Z}$.
	The set of solution of $\VI(\Z,F)$ is denoted by $\SOL(\Z,F)$.
\end{definition}
By "equivalent", we mean specifically that $\SOL(\X,v)$ coincides with the NE of $\Gamma$ (see e.g. Corollary 1 in \cite{Facchinei2007}). The definition of VI asserts that the agents gradients points toward $\SOL(\X,v)$ and therefore \eqref{Eq:aoaaosjsdhdjdhhddss2} converges (under certain condition on the learning rate) to $\SOL(\X,v)$ \cite{Mertikopoulos2018}. However elements of $\SOL(\X,v)$ do not necessarily satisfies the resource constraints $\mathcal{Q}$. Consequently we cannot expect resource sustainability behaviour of the population applying \eqref{Eq:aoaaosjsdhdjdhhddss2}.

In order to handle this issue one may consider instead $\SOL(\mathcal{Q},v)$, which is non-empty since $\mathcal{Q}$ is convex and compact and $v$ is monotone, and tries to ensure \eqref{Eq:aoaaosjsdhdjdhhddss2} (possibly by some modifications) to converges to this set. This procedure is the basic of our approach.

\section{Price Mechanism}
\label{Sec:aoosjsjshhdjdggdhdhddd}
In order to give the agents incentives for sustainable use of resources, our advice is to charge each agent additional cost for the amount of utilization of resources related to her action. Specifically, consider a time slot $t$, agent $i$ is obligate to pay $\tilde{\Lambda}_{t}^{\T}\mathbf{A}_{i}X_{t+1}^{(i)}$ for a possible future action $X^{(i)}_{t+1}\in\mathcal{X}_{i}$, where $\tilde{\Lambda}_{t}$ is a vector specifying the price of each resource at time $t$. So at time $t$, the utility function of agent $i$ becomes $u^{(i)}_{t}(\cdot)+\tilde{\Lambda}_{t}^{T}\mathbf{A}_{i}(\cdot)$,
and correspondingly assuming that the price information is not noisy, the gradient update \eqref{Eq:aoaaosjsdhdjdhhddss2} turns to:
\begin{equation}
\label{Eq:GradStep2}
Y_{t+1}^{(i)}=Y_{t}^{(i)}+\gamma_{t}(\hat{v}_{t}^{(i)}+\mathbf{A}_{i}^{\T}\tilde{\Lambda}_{t}).
\end{equation}    
The update of each entry of the price vector $\tilde{\Lambda}_{t+1}$ is done by each of the resources separately proportional to their own congestion state.
The specific mechanism is provided in Algorithm \ref{Alg:aoaishhjddhhddddeee}.
\begin{algorithm}
	\caption{}
	\begin{algorithmic}
		\REQUIRE Horizon length $T$,
		\begin{itemize}
			\item For each $t\in [T-1]_{0}$: agents' learning rate $\gamma_{\tau}>0$, resources' learning rate $\zeta_{\tau}\in(0,1)$, price progressivity $\eta_{\tau}\geq 0$, price sensitivity $\beta_{\tau}\geq 0$,
			\item Initialization: Score vectors $Y_{0}^{(i)}\in\real^{D_{i}}$, $i\in [N]$, prices $\Lambda^{r}_{0}\in\real_{\geq 0}$, $r\in [R]$.
		\end{itemize}
		\begin{center}
		\textbf{//Mechanism}
		\end{center}
		\FOR{$t=1,2,\ldots,T$}
		\STATE Every agent $i\in [N]$ mutually play $X^{(i)}_{t}\gets\Phi_{i}(Y^{(i)}_{t})$
	    \STATE \textbf{\emph{//Decision making via Online Learning}}
		\FOR{every player $i\in [N]$}
%
\STATE Observe noisy gradient utility feedback  \eqref{Eq:NoiseGrad} and update the score vector $Y^{(i)}_{t+1}$ via \eqref{Eq:GradStep2}
		\STATE Query the prices $\tilde{\Lambda}^{r}_{t}$ from the resources $r\in [R]$
		\ENDFOR
		\STATE \textbf{\emph{//Pricing}}
		\FOR{every resource $r\in [R]$}
		\STATE Check the actual own congestion state: $\phi_{t}^{r}=\phi^{r}(X_{t})=[\mathbf{A}X_{t}-b]_{r}$
		\STATE Update the price: $\Lambda_{t+1}^{r}\leftarrow\left[(1-\eta_{t})\Lambda_{t}^{r}+\zeta_{t}\phi^{r}_{t}\right]_{+}$ and $\tilde{\Lambda}_{t+1}^{r}\leftarrow\beta_{t}\cdot \Lambda_{t+1}^{r}$ 
		\ENDFOR 
		\ENDFOR
	\end{algorithmic}
	\label{Alg:aoaishhjddhhddddeee}
\end{algorithm}

The parameter $\beta_{\tau}$ specifies to what extent the price of a resource should be considered in the decision-making process of the agents. In order to understand the effect of this parameter to the population dynamic, let us consider the extreme cases $\beta_{\tau}=0$ and high $\beta_{\tau}>0$. With $\beta_{\tau}=0$, the population dynamic described in Algorithm \ref{Alg:aoaishhjddhhddddeee} turns to \eqref{Eq:aoaaosjsdhdjdhhddss2}. As already outlined in Subsection \ref{SubSec:aoosjsjdhhdjdhhsggshhsgss}, we cannot expect decaying resource congestion. Now, if $\beta_{\tau}>0$ is high, the agents tend to take the action with cheapest cost. Since the price of a resource is proportional to its congestion state, all agents might at worst (e.g. in the case $N=D$ and $\mathbf{A}_{i}=\mathbf{I}_{N}$, where an action corresponds directly to resource utilization choice) fully consume a single resource with the lowest congestion and cause therefore the latter's price and load to rise dramatically. Subsequently in the next time slot, they will all mutually fully utilized another less congested and cheaper resource causing its price and congestion to rise dramatically. This procedure will repeat, cause agents' consumption choice bounces at worst from a single resource to another one, and meanwhile violation of resource capacity constraints. This gedankenexperiment asserts in particular that high prices and thus high degree of control, in contrary to the intuition, does in general not support sustainable behaviour. Rather, one should allow for the latter's sake to a certain degree egoistic behaviour of the agents.
%

The parameter $\eta_{\tau}$ specifies the strength of the dependency of the price update on the previous price. $\eta_{t}=1$ corresponds to the extreme case where the price update only based on the actual congestion state $\phi_{t}^{r}$. We expect $\eta_{\tau}=1$ is not a good choice since it ignores the price dynamic and correspondingly the agents' consumption behaviour implicitly described therein. A problem which might occur with the extreme case $\eta_{\tau}=0$ is the rapid increase of the prices causing the price update insensitive against changes in the congestion state of the resources.


%

\section{Non-asymptotic Guarantee of the Price Mechanism}
In this section we provide a theoretical analysis of the price mechanism provided in Algorithm \ref{Alg:aoaishhjddhhddddeee}. Our emphasize is on the degree of its contribution to the resource-aware consumption behaviour of the agents, which we measured by the (time) \emph{average of the norm of the clipped cumulative violation of constraints (ANCCVC)}:
\begin{equation*}
\ANCCVC_{t}:=\tfrac{\Erw\left[ \norm{\left[ \sum_{t=0}^{t-1} (\mathbf{A}X_{t}-\mathbf{b})\right]_{+}}_{2}\right]}{t},\quad t\in\nat.
\end{equation*}
ANCCVC gives in particular an estimate for the time average congestion state of the resource since $\ANCCVC_{t}\geq\sum_{\tau=0}^{t-1}\phi_{r}(X_{\tau})/t$ for all $r\in [R]$. 

Throughout, $C_{1},C_{2},C_{3}$ denote non-negative constants fulfilling for all $x\in\mathcal{X}$ and $\lambda\in\real^{M}_{\geq 0}$:
\begin{equation}
\label{Eq:aoaoshshjhsjjss}
\norm{\mathbf{A}^{T}\lambda}_{*}\leq C_{1}\norm{\lambda}_{2},~\norm{v(x)}_{*}\leq C_{2}, \norm{g(x)}_{2}\leq C_{3},
\end{equation}  
which clearly exists by our assumptions on $u$ and $\mathcal{X}$. 
Our main result is the following:
\begin{theorem}
	\label{Eq:aoaosjjsdhdjjdhhdjjddd}
Given a horizon length $t\in\nat$ and learning rate $\gamma_{\tau}>0$, $\tau\in [t-1]_{0}$. Set the extrinsic price sensitivity of the agents as $\beta_{\tau}=2$ and $\zeta_{\tau}=\gamma_{\tau}$, for all $\tau\in [t-1]_{0}$ and suppose that for all $\tau\in [t-1]_{0}$ the agents' learning rate and the price progressivity fulfills:
\begin{equation}
\label{Eq:aiaiahhfggfdhdgdhdhdgsgss}
\eta_{\tau}^{2}-\tfrac{\eta_{\tau}}{4}+\tfrac{\gamma_{\tau}^{2}C_{1}^{2}}{4K}\leq 0\quad\quad\quad\text{(Trackability Condition (TC))},
\end{equation}
Then it holds for $\Lambda_{0}=0$ and $Y_{0}=0$:
\begin{equation}
\label{Eq:aasisihdhdgdhsgsgssss}
\Erw\left[ \norm{[\sum_{\tau=0}^{t-1}\gamma_{\tau}(\mathbf{A}X_{\tau}-b)]_{+}}_{2}^{2}\right] \leq 2\overline{\eta}_{t} \left(\Delta\psi
+\tilde{C}_{1}\sum_{\tau=0}^{t-1}\gamma^{2}_{\tau} \right)+ \overline{\eta}_{t}^{2}\left( \tilde{C}_{2}^{2}
+ \tfrac{4}{K}\sum_{\tau=1}^{t}\gamma_{\tau-1}^{2}\Erw[\norm{\xi_{\tau}}_{*}^{2}]\right), 
\end{equation}
where:
\begin{equation*}
\overline{\eta}_{t}:=\sum_{\tau=0}^{t-1}\eta_{\tau}+1,~\tilde{C}_{1}:=2\left(\tfrac{C_{2}^{2}}{K}+2C_{3}^{2}\right),~
\Delta\psi=\sum_{i=1}^{N}\left( \max_{\X_{i}}\psi_{i}-\min_{\X_{i}}\psi_{i}\right),~K:=\min_{i}K_{i},
\end{equation*}
and $\tilde{C}_{2}>0$ is a constant independent of $t$, the choice of mirror maps $\Phi_{i}$, $i\in [N]$, and the noises $(\xi_{\tau})_{\tau\in [T]}$.
\end{theorem}
The proof of this theorem is rather technical and is provided in Appendix \ref{Sec:aiaishshhddjjdjddsss}. It is based on the analysis of the dynamic of the energy function $\mathcal{E}_{t}((x,\lambda),\tilde{\lambda}):=\mathcal{E}^{1}_{t}((x,\lambda))+\mathcal{E}^{2}_{t}(\tilde{\lambda})$
where $\mathcal{E}^{1}_{t}((x,\lambda)):=F(x,Y_{t})+\tfrac{\norm{\Lambda_{t}-\lambda}_{2}^{2}}{2}$, $\mathcal{E}^{2}_{t}(\tilde{\lambda}):=\tfrac{\norm{\Lambda_{t}-\tilde{\lambda}}_{2}^{2}}{2}$, $x\in\mathcal{X}$, and $\lambda,\tilde{\lambda}\in\real^{R}$.
We use as the reference point for $\mathcal{E}^{1}_{t}$ the strategy $x\in\SOL(\mathcal{Q},v)$. Such $x$ possesses desired sustainable property and attracts $X_{t}$ (see Subsection \ref{SubSec:aoosjsjdhhdjdhhsggshhsgss}). The latter is ensured by additionally choosing $\lambda$ as the dual point in the KKT-system of $\SOL(\mathcal{Q},v)$. $\mathcal{E}^{2}_{t}$ with suitable $\tilde{\lambda}$ provide more specific information about the norm of the clipped weighted cumulative resource congestion.

The energy function $\mathcal{E}_{t}$ contrasts to that used in \cite{Mertikopoulos2018} which is $F(x,Y_{t})$. Furthermore, notice that since the dual constraint space $\real^{R}_{\geq 0}$ is not compact, $\mathcal{E}^{1}_{t}$ is not merely a trivial extension of the energy function $F(x,Y_{t})$ used in \cite{Mertikopoulos2018} yielding by seeing the dual variable as a new player. Furthermore, the energy function $\mathcal{E}_{t}^{1}$ is used in the literature of the constrained online optimization \cite{Mahdavi1}. In the game setting, it seems necessary to involve the additional energy function $\mathcal{E}_{2}(\Lambda_{t},\tilde{\lambda})$ and set the price sensitivity to $2$ in order to obtain the bound given in above theorem.


\textbf{On Trackability Condition:}
In order that \eqref{Eq:aiaiahhfggfdhdgdhdhdgsgss} is fulfilled at a time $\tau$, it is necessary that $\eta_{\tau}^{2}-(\eta_{\tau}/4)<0$. Therefore, the requirement \eqref{Eq:aiaiahhfggfdhdgdhdhdgsgss} demands that  $\eta_{\tau}<1/4$. This observation gives the advice to the resources not to be fully progressive in the price determination , i.e. to avoid the parameter $\eta_{\tau}\approx 1$. By attempting to solve the quadratic inequality  \eqref{Eq:aiaiahhfggfdhdgdhdhdgsgss} one can see that a necessary condition on $\gamma_{\tau}$ s.t. \eqref{Eq:aiaiahhfggfdhdgdhdhdgsgss} holds at time $\tau$ is $\gamma\leq\tfrac{1}{4 C_{1}}\sqrt{\tfrac{K}{2}}$
In this case, \eqref{Eq:aiaiahhfggfdhdgdhdhdgsgss} is equivalent to:
\begin{equation*}
\tfrac{\frac{1}{4}-\sqrt{\frac{1}{16}-\frac{\gamma_{\tau}^{2}C^{2}_{1}}{K}}}{2}\leq\eta_{\tau}\leq \tfrac{\frac{1}{4}+\sqrt{\frac{1}{16}+\frac{\gamma_{\tau}^{2}C^{2}_{1}}{K}}}{2}.
\end{equation*}
This observation assert that for small $\gamma_{\tau}$, one can choose $\eta_{\tau}$ approximately in the interval $(0,1/4)$. 
%
%
%
%
%
\begin{remark}
Suppose that $\gamma_{\tau}=C_{\gamma}\tau^{-p}$ for a certain $C_{\gamma}$ and $p>0$, and $\eta_{\tau}=C_{\eta} \tau^{-q}$ for a certain $C_{\eta}>0$ and $q>0$. In order that TC holds, it is necessary that $\eta_{\tau}$ decays with the same order like or slower than $\gamma_{\tau}^{2}$. Therefore we have to require $q\in (0,2p]$.  
\end{remark}

Now we are ready to give several consequences of Theorem \eqref{Eq:aoaosjjsdhdjjdhhdjjddd}. For simplicity, we assume that the noise is persistent, in the sense that $\Erw[\norm{\xi_{\tau}}^{2}_{*}]\leq \sigma_{*}^{2}$, for all $\tau\in\nat$.

\textbf{Constant Learning rate:}
Let us consider a finite time horizon $T\in\nat$ and $\beta_{\tau}=2$, for all $\tau\in [T-1]_{0}$. Furthermore, let us consider the case where both, the learning rate of the agents and the price progressivity are constant, i.e. $\gamma_{\tau}=\gamma$ and $\eta_{\tau}=\eta$, for all $\tau\in [T-1]_{0}$. Assuming that $\gamma$ and $\eta$ fulfills \eqref{Eq:aiaiahhfggfdhdgdhdhdgsgss} holds, it follows from \eqref{Eq:aasisihdhdgdhsgsgssss} and Jensen's inequality:
\begin{equation}
\label{Eq:aaiissjdhhdjshhsshhdjjdddd}
\Erw[\ANCCVC_{T}]\leq\sqrt{ \tfrac{2(\eta T +1)}{\gamma T^{2}}\left(  \tfrac{\Delta\psi}{\gamma}
	+\tilde{C}_{1}\gamma T\right) + \tfrac{(\eta T+1)^{2}}{\gamma^{2}T^{2}}\tilde{C}^{2}_{2}
	+ \tfrac{(\eta T+1)^{2}}{T^{2}}\tfrac{4\sigma_{*}^{2}T}{K}}.
\end{equation}
So, suppose that $\gamma=\Theta(T^{-p})$ with $p\in[1/2,1)$. 
Setting $\eta=\Theta(T^{-q})$ where $q\in(1/2,2p]$,
it yields:
\begin{equation}
\label{Eq:aaiisjshhdjjdhhddddfffdd3}
\Erw[\ANCCVC_{T}]\leq \mathcal{O}\left(T^{p-\frac{q+1}{2}}+ T^{p-q}+\sigma_{*} T^{\frac{1}{2}-q} \right) .
\end{equation}
In particular if $p=1/2$, we can choose $q=1$ and obtain a sub-linear bound for the ANCCVC at time $T$ of order $\mathcal{O}((1+\sigma_{*})T^{-\frac{1}{2}})$.
\paragraph{Variable Parameters:} If the agents are each willing to apply the ergodic average of their historical strategies instead of their actual strategies, we can ensure the decay of the violation of resource constraints with time in expectation. In order to show this, let us consider the infinite time horizon $T=\infty$. We use Jensen's inequality to obtain the following bound from \eqref{Eq:aasisihdhdgdhsgsgssss}:
\begin{align*}
\Erw\left[ \norm{[\sum_{\tau=0}^{t-1}\gamma_{\tau}(\mathbf{A}X_{\tau}-b)]_{+}}_{2}\right] \leq &\sqrt{2\overline{\eta}_{t} \left(\Delta\psi
+\tilde{C}_{1}\sum_{\tau=0}^{t-1}\gamma^{2}_{\tau} \right)+ \overline{\eta}_{t}^{2}\left( \tilde{C}_{2}^{2}
+ \tfrac{4\sigma_{*}^{2}}{K}\sum_{\tau=0}^{t-1}\gamma_{\tau}^{2}\right)},\quad\forall t\in\nat
\end{align*}
For the ergodic average $\overline{X}^{\gamma}_{t}=\tfrac{\sum_{\tau=0}^{t-1}\gamma_{\tau}X_{\tau}}{\sum_{\tau=0}^{t-1}\gamma_{\tau}}$ of the population iterate, we have:
\begin{equation*}
\label{Eq:aiaishshsddgdgdhhddd}
\Erw\left[ \norm{[\mathbf{A}\overline{X}^{\gamma}_{t}-b]_{+}}_{2}\right]
 \leq \tfrac{\sqrt{2\overline{\eta}_{t} \left(\Delta\psi
	+\tilde{C}_{1}\sum_{\tau=0}^{t-1}\gamma^{2}_{\tau} \right)+ \overline{\eta}_{t}^{2}\left( \tilde{C}_{2}^{2}
	+ \tfrac{4\sigma_{*}^{2}}{K}\sum_{\tau=0}^{t-1}\gamma_{\tau}^{2}\right)}}{\sum_{\tau=0}^{t-1}\gamma_{t}}, 
\end{equation*}
Setting $\gamma_{t}=\Theta(t^{-1/2})$ and $\eta_{t}=\Theta(t^{-1})$ fulfilling trackability condition,
it follows that the decay of the congestion state is in the noiseless case of order $\mathcal{O}(\ln(t)/\sqrt{t})$ and otherwise $\mathcal{O}(\ln^{3/2}(t)/\sqrt{t})$.
%
Now let be $\gamma_{t}=\Theta(t^{-1})$ and $\eta_{t}=\Theta(t^{-2})$, we have decay of order $\mathcal{O}((1+\sigma_{*})/\ln(t))$.
\section{Numerical Experiment}

\textbf{Exponential Weights Online Learning in Quadratic Game:}
We consider $N$ agents whose task is to allocate a certain amount of tasks to $R$ resources. The strategy space of agent $i$ corresponds to the simplex $\Delta:=\lrbrace{x^{(i)}\in\real^{R}_{\geq 0}:\sum_{r=1}^{R}x_{k}^{(i)}=1}$.
For a strategy $x^{(i)}\in\Delta$, $x_{r}^{(i)}$ stands for the proportion of tasks agent $i$ assigns to resource $r\in [R]$. The cost function of player $i$ is quadratic and given by $J^{(i)}(x^{(i)},x^{(-i)})=\tfrac{1}{2}\inn{x^{(i)}}{Q x^{(i)}}+\inn{C\sigma(x)+c^{i}}{x^{(i)}}$, where $\sigma(x)=\tfrac{1}{N}\sum_{i=1}^{N}x^{(i)}$
where $c_{i}\in\real^{D}$, $Q\in\real^{D\times D}$ and $C\in\real^{D\times D}$ are positive semi-definite, and either $Q$ or $C$ are positive definite. In order to apply our method, we set $u^{(i)}(x)=-J^{(i)}(x)$. The corresponding gradient mapping is given by $v(x)=-\left[ (I_{N}\otimes Q+ \frac{1}{N}\mathbf{1}_{N}\mathbf{1}_{N}^{\T}\otimes C)x+c+\frac{1}{N}(I_{N}\otimes C^{\T})x\right]$,
where $\otimes$ denotes the Kronecker product between two matrices. For the mirror map of the agents, we use the logit choice implemented by means of log-sum trick in order to avoid numerical overflow.

\textbf{Game Parameter:}
We consider $N=20$, $D=R=5$, and $T=500$, and study the case where parameters are fixed. We set $Q=2\sqrt{\tilde{Q}^{\T}\tilde{Q}}+\mathbf{I}_{D}$, where the entries of $\tilde{Q}$ is chosen independently normal distributed. Moreover we consider the case where  $C=4\mathbf{I}_{D}$, $c=0$, $\mathbf{A}=4\mathbf{1}_{N}^{\T}\otimes\mathbf{I}_{D}$, and $b=16\mathbf{1}_{D}$. For specific model of the stochastic feedback we use Gaussian vector with covariance matrix $\sigma^{2}\mathbf{I}_{D}$, where $\sigma>0$.

\textbf{Evaluation:}
For the experiments each depicted in Figures \ref{SubFig:beta}, \ref{SubFig:alpha}, and \ref{SubFig:gamma} we use different realization of the noise with same power $\sigma=5$.  
\begin{figure}
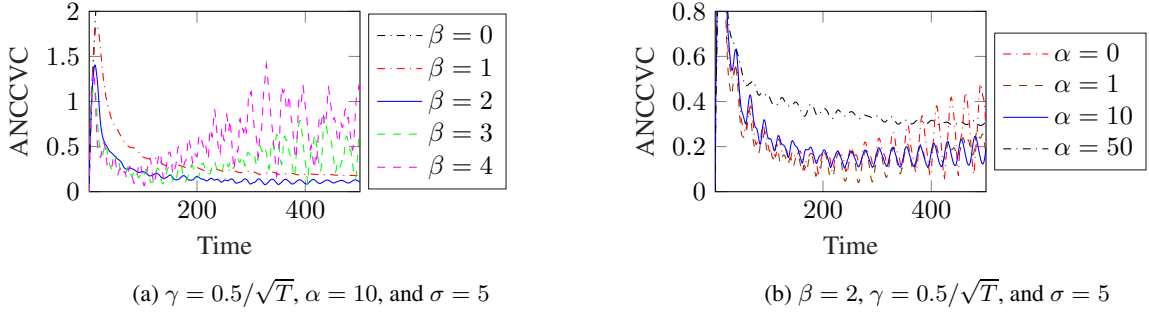

	\begin{subfigure}[c]{0.5\textwidth}
		
		\input{PlotNIPS1}
		\subcaption{$\gamma=0.5/\sqrt{T}$, $\alpha=10$, and $\sigma=5$}
		\label{SubFig:beta}
		
	\end{subfigure}
	\begin{subfigure}[c]{0.5\textwidth}
		
		\input{PlotNIPS2}
		\subcaption{$\beta=2$, $\gamma=0.5/\sqrt{T}$, and $\sigma=5$}
		\label{SubFig:alpha}
	\end{subfigure}

\caption{Dynamic of ANCCVC of Algorithm \ref{Alg:aoaishhjddhhddddeee} for different price sensitivities $\beta$ and - progressivities $\eta=\alpha\gamma^{2}$ }
\label{Fig:plot1}
\end{figure}
Figure \ref{SubFig:beta} shows that pure egoistic uncontrolled behaviour of the agents ($\beta=0$) may lead to immense overuse of the resources, and that control of agents' consumption via price mechanism ($\beta>0$) can prevent this event. 
With price regularization ($\beta>0$), we observe the tendency of oscillation in the agents' dynamic, whereby the following difference is observable: The choices $\beta=1$ and $\beta=2$ effect in stabilizing behaviour, while $\beta=3$ and $\beta=4$ effect in chaotic behaviour. This observations are aligned with the gedankenexperiment done in Section \ref{Sec:aoosjsjshhdjdggdhdhddd}, one of whose conclusions is that high price sensitivity might cause the agents' utilization strategies to mutually bouncing between single resources. Furthermore, Figure \ref{SubFig:beta} confirms the optimality of the selection of parameter choice $\beta=2$ given in Theorem \ref{Eq:aoaosjjsdhdjjdhhdjjddd}, since it tends to have the lowest ANCCVC. 
From Figure \ref{SubFig:alpha}, we can observe that in the non-progressive case $\alpha=0$ ($\eta=\alpha\gamma^{2}=0$), the corresponding dynamic of ANCCVC resonates heavily and possess at the end of the time-horizon ($t=500$) highest value (aside from $\alpha=50$). This asserts the importance of progressivity in the price determination and also justifies the importance of the TC \eqref{Eq:aiaiahhfggfdhdgdhdhdgsgss}. The parameters $\alpha=1$ and $\alpha=10$ has the best behaviour in this experiment. We see the tendency of decreasing oscillation with increasing price progressivity. However by observing $\alpha=10$ the overall performance at the end of the time horizon might be worse for high $\eta$. This observation underlines the role of $\eta$ as the parameter specifying the decay rate of ANCCVC asserted by the bound \eqref{Eq:aaiissjdhhdjshhsshhdjjdddd}. From Figure \ref{SubFig:gamma} the high oscillatory chaotic behaviour of $\gamma=1/\sqrt{T}$ and $\gamma=10/\sqrt{T}$ is aligned with the TC \eqref{Eq:aiaiahhfggfdhdgdhdhdgsgss} which eliminate the possibility that for fixed $\eta$, $\gamma$ can be arbitrarily high. Moreover the plot for $\gamma=0.05/\sqrt{T}$ shows that too small $\gamma$ caused slow decay of the ANCCVC as predicted by the bound \eqref{Eq:aaiissjdhhdjshhsshhdjjdddd}.   

\begin{figure}
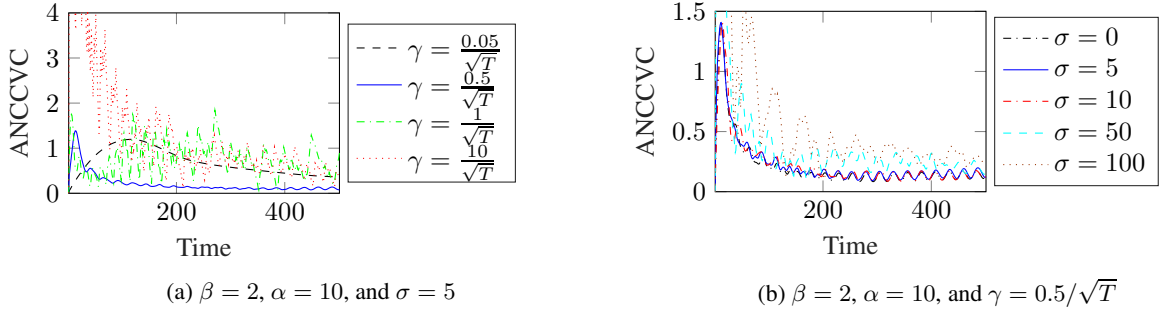

	\begin{subfigure}[c]{0.5\textwidth}
		
		\input{PlotNIPS3}
		\subcaption{$\beta=2$, $\alpha=10$, and $\sigma=5$}
		\label{SubFig:gamma}
	\end{subfigure}
	\begin{subfigure}[c]{0.5\textwidth}
		
		\input{PlotNIPS4}
		\subcaption{$\beta=2$, $\alpha=10$, and $\gamma=0.5/\sqrt{T}$}
		\label{SubFig:sigma}
	\end{subfigure}
	
	\caption{Dynamic of ANCCVC of Algorithm \ref{Alg:aoaishhjddhhddddeee} for different learning rates $\gamma$ and noise powers $\sigma$.}
\end{figure} 
\label{Fig:plot2}

Figure \ref{SubFig:sigma} shows that the noise power has no significant influence to the ANCCVC. This observation is somehow forecasted by our theoretical results since the noise term in the corresponding expectation bound decay with square roots of the time and the noise is light-tailed.
\section{Concluding Remarks}
Although machine learning methodology becomes evermore present in the modern technology and thus in our increasingly technologized daily life, most of the learning literature only concern with the optimal behaviour of a single learner independently from the possible impact of her decision. Motivated by this issue and everlasting resource scarcity problem, we took a game-theoretic macroscopic view of the online learning paradigm and extended the resulted model by concerning possible impact of the learners' decision to the scarce resources. This allowed us to design a control method via decentralized pricing which provides the learners incentives for sustainable behaviour. In the best case we can ensure the decay of the time average accumulation of the resource congestion of order $\mathcal{O}(T^{-1/2})$ with time horizon $T$, which also holds in expectation in case that the learners have noisy first-order feedback. Moreover if the online learners are willingly to apply the ergodic average - instead of their OMD-based strategy itself, we can ensure in the best case the decay of the resource congestion with time $t$ of order $\mathcal{O}(\ln(t)/\sqrt{t})$ in the noiseless case and in expectation of order $\mathcal{O}(1/\ln(t))$.

Nevertheless, there are several interesting questions remains open such as: whether the guarantees holds also with high probability; whether the quality of the pricing mechanism remains, if asynchronity in the agent update holds, be it in the choice of learning rate, or in the update time, or if mismatch between resources' - and agents' learning rate.

\bibliographystyle{unsrt}

%
%
%
%
%
\appendix

\section{Plots}
\begin{figure}[h]
		\input{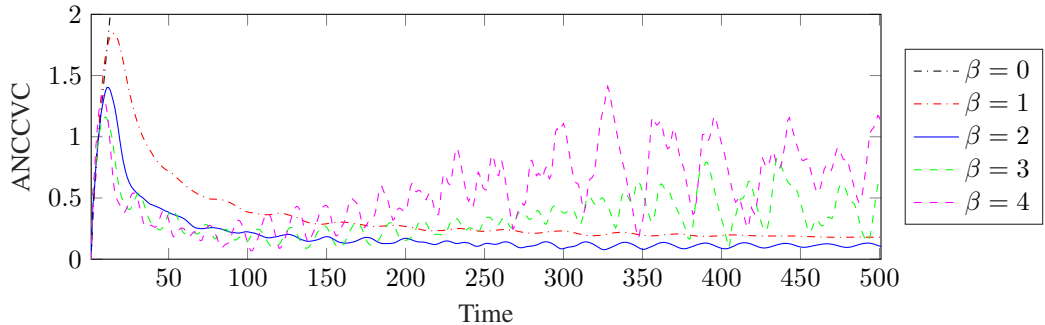}
		\caption{$\gamma=0.5/\sqrt{T}$, $\alpha=10$, and $\sigma=5$}
		\label{SubFig:betaL}	
\end{figure}
\begin{figure}[h]
	\input{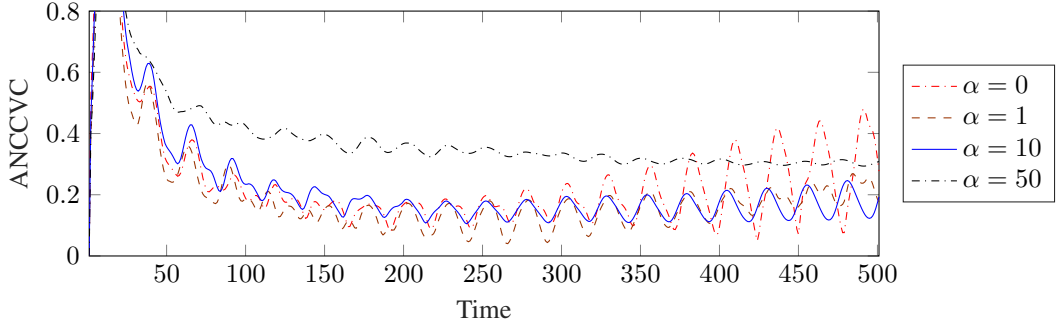}
	\caption{$\beta=2$, $\gamma=0.5/\sqrt{T}$, and $\sigma=5$}
\label{SubFig:alphaL}
\end{figure}
\begin{figure}[h]
	\input{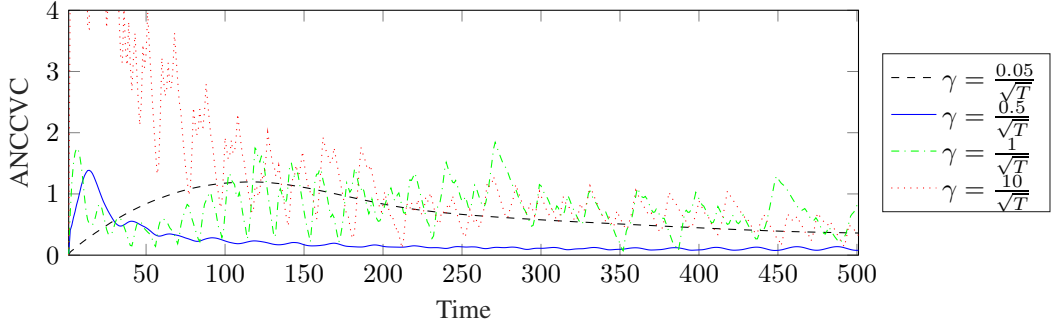}
		\caption{$\beta=2$, $\alpha=10$, and $\sigma=5$}
	\label{SubFig:gammaL}	
\end{figure}
\begin{figure}[h]
	\input{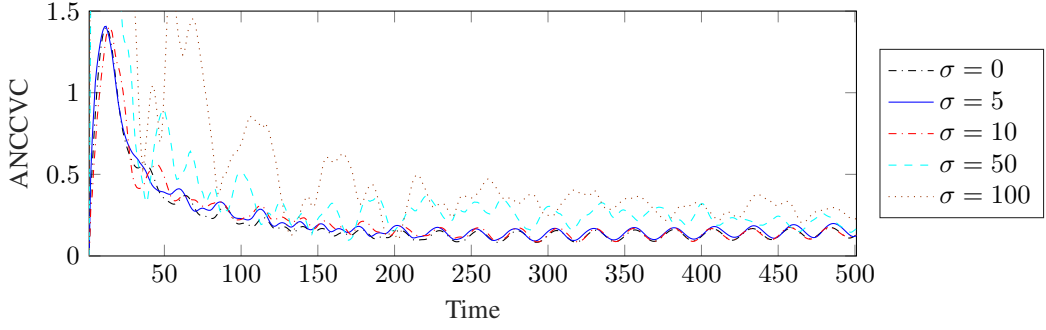}
		\caption{$\beta=2$, $\alpha=10$, and $\gamma=0.5/\sqrt{T}$}
	\label{SubFig:sigmaL}	
\end{figure}

\section{Proof of the Theorem 1}
\label{Sec:aiaishshhddjjdjddsss}
The energy function $F(x,Y_{t})$ is used in the literature of multi-agent online learning (see e.g. \cite{Mertikopoulos2018}). As done in \cite{Mertikopoulos2018} by analyzing the dynamic of $F(x_{*},Y_{t})$, $x_{*}\in\SOL(v,\mathcal{X})$ is a candidate for a population strategy, to which the multi-agent OMD iterate $X_{t}$ given by \eqref{Eq:aoaaosjsdhdjdhhddss2} converges. 

Following this approach, we analyze in the following $F(x_{*},Y_{t})$, where $x_{*}\in\SOL(v,\mathcal{Q})$. If it then turns out that $X_{t}$ converges to $x_{*}$, we can be sure that at least $X_{t}$ asymptotically fulfills the resource constraints which aligns with our aim. Toward this end, we have:
\begin{lemma}
	\label{Lem:aaosshdjhdgdhhsgsgss}
	For all $t\in\nat$ and $x\in\mathcal{X}$, it holds for $\mathcal{V}_{t}^{(1)}(x):=F(x,Y_{t})-F(x,Y_{0})$: 
	\begin{align*}
	\mathcal{V}_{t}^{(1)}(x)&\leq\sum_{\tau=0}^{t-1}\gamma_{\tau}\inn{X_{\tau}-x}{v(X_{\tau})}-\sum_{\tau=0}^{t-1}\gamma_{\tau}\inn{X_{\tau}-x}{\beta_{\tau}\mathbf{A}^{\T}\Lambda_{\tau}}
	\\
	&+\sum_{\tau=0}^{t-1}\tfrac{\beta_{\tau}^{2}\gamma_{\tau}^{2}C_{1}^{2}}{K}\norm{\Lambda_{\tau}}^{2}_{2}+ S_{t}(x)+\tfrac{2}{K}R_{t}+\tfrac{2 C_{2}^{2}\sum_{\tau=0}^{t-1}\gamma_{\tau}^{2}}{K},
	\end{align*}
	where $C_{1}$ and $C_{2}$ are given in \eqref{Eq:aoaoshshjhsjjss}, and where:
	\begin{equation*}
	S_{t}(x):=\sum_{\tau=0}^{t-1}\gamma_{\tau}\inn{X_{\tau}-x}{\xi_{\tau+1}},~R_{t}:=\sum_{\tau=0}^{t-1}\gamma_{\tau}^{2}\norm{\xi_{\tau+1}}^{2}_{*}.
	\end{equation*}
\end{lemma}
Since $X_{\tau}$ is not necessarily in $\mathcal{Q}$, it is not yet obvious that $X_{t}$ receive negative drift toward $x_{*}\in\SOL(v,\mathcal{Q})$ ($\SOL(v,\mathcal{Q})\neq\emptyset$ since $v$ is monotone and $\mathcal{Q}$ is compact) by observing the term $\inn{X_{\tau}-x_{*}}{v(X_{\tau})}$. In order to adress this issue, we express $\SOL(v,\mathcal{Q})$ equivalently in the higher space $\real^{D}\times\real^{R}$ by means of Lagrangian method. Usual KKT argumentation give the following statement:
\begin{proposition}
	\label{Prop:aiaisshshjdhdggddd}
	Suppose that Assumption \ref{Eq:aioaoosjjshhdjjddss} holds. Then the following statements are equivalent:
	\begin{enumerate}
		\item $\overline{x}\in\mathcal{Q}$ is a solution of $\VI(\mathcal{Q},v)$
		\item There exists $\overline{\lambda}\in\real^{R}_{\geq 0}$ s.t. $(\overline{x},\overline{\lambda})$ is a solution of $\VI(\mathcal{X}\times \real^{R}_{\geq 0},\tilde{v})$, where:
		 \begin{equation}
		 \label{Eq:iaaiisjsjdhhdhdhdhdjjsss8}
		 \tilde{v}:\mathcal{X}\times \real^{R}_{\geq 0}\rightarrow \real^{D+R},\quad (x,\lambda)\mapsto \left[\begin{array}{c}
		 v(x)-\mathbf{A}^{\T}\lambda\\
		 \mathbf{A}x- b
		 \end{array}
		 \right].
		 \end{equation}
	\end{enumerate}
\end{proposition}   
In order to benefit from above proposition, we now analyze the dynamic of the energy function $\norm{\Lambda_{t}-\lambda}_{2}^{2}$:
\begin{lemma}
	\label{Lem:aaosshdggdhsgsghdgdgddd}
	For all $t\in\nat$ and $\lambda\in\real^{R}_{\geq 0}$, it holds for $\mathcal{V}_{t}^{(2)}(\lambda):=(\norm{\Lambda_{t}-\lambda}_{2}^{2}-\norm{\Lambda_{0}-\lambda}_{2}^{2})$:
	\begin{align}
	\mathcal{V}_{t}^{(2)}(\lambda)&\leq \sum_{\tau=0}^{t-1}\zeta_{\tau}\inn{\Lambda_{\tau}-\lambda}{\mathbf{A}X_{\tau}-b}+\sum_{\tau=1}^{t-1}\left( 2\eta_{\tau}^{2}-\tfrac{\eta_{\tau}}{2}\right)\norm{\Lambda_{\tau}}^{2}_{2}+\norm{\lambda}_{2}^{2}\sum_{\tau=0}^{t-1} \tfrac{\eta_{\tau}}{2} +2C_{3}^{2}\sum_{\tau=0}^{t-1}\zeta_{\tau}^{2}.\label{Eq:aaiaiajssjshhdddd}
	\end{align}
\end{lemma}
Combining the previous bounds for $\mathcal{V}_{t}^{(1)}(x)$ and $\mathcal{V}_{t}^{(2)}(\lambda)$, it yields for $z=(x,\lambda)\in\mathcal{\X}\times\Lambda$:
\begin{align*}
\mathcal{V}_{t}(z)\leq&-\sum_{\tau=0}^{t-1}\gamma_{\tau}\Theta_{\tau}(z)+\sum_{\tau=0}^{t-1}\gamma_{\tau}(1-\beta_{\tau})\inn{X_{\tau}-x}{\mathbf{A}^{\T}\Lambda_{\tau}}+\tfrac{2C_{2}^{2}}{K}\sum_{\tau=1}^{t-1}\gamma^{2}_{\tau}+2C_{3}^{2}\sum_{\tau=1}^{t-1}\zeta^{2}_{\tau}\\
&+\norm{\lambda}_{2}^{2}\sum_{\tau=0}^{t-1}\tfrac{\eta_{\tau}}{2}
+ \sum_{\tau=0}^{t-1}\left( 2\eta_{\tau}^{2}-\frac{\eta_{\tau}}{2}+\frac{\beta_{\tau}^{2}\gamma_{\tau}^{2}C_{1}^{2}}{K}\right)\norm{\Lambda_{\tau}}^{2}_{2}
+S_{t}(x)+\tfrac{2}{K}R_{t}\\
&+\sum_{\tau=0}^{t-1}(\gamma_{\tau}-\zeta_{\tau})\inn{\Lambda_{\tau}-\lambda}{\mathbf{A}X_{\tau}-b}\nonumber,
\end{align*}
where $\mathcal{V}_{t}(z)=\mathcal{V}_{t}^{(1)}(x)+\mathcal{V}_{t}^{(2)}(\lambda)$, and where:
\begin{equation*}
\Theta_{t}(z):=\inn{z-Z_{t}}{\tilde{v}(Z_{t})}.
\end{equation*}
By straightforward computation one can show that $v$ monotone asserts $\tilde{v}$ is monotone. Thus it holds:
\begin{equation*}
\Theta_{t}(z)\geq\inn{z_{*}-Z_{t}}{\tilde{v}(z_{*})}\geq 0,\quad\forall z_{*}\in\SOL(\mathcal{X}\times\real^{R}_{\geq 0},\tilde{v}).
\end{equation*}
This and the choice $\zeta_{\tau}=\gamma_{\tau}$ yields for $z_{*}=(x_{*},\lambda_{*})\in \SOL(\mathcal{X}\times\real^{R}_{\geq 0},\tilde{v})$ :
\begin{equation}
\label{Eq:aaosjsjshdhfggfhhdggddhhsssaa}
\begin{split}
\mathcal{V}_{t}(z_{*})\leq&\sum_{\tau=0}^{t-1}\gamma_{\tau}(1-\beta_{\tau})\inn{X_{\tau}-x_{*}}{\mathbf{A}^{\T}\Lambda_{\tau}}+\tfrac{2C_{2}^{2}}{K}\sum_{\tau=1}^{t-1}\gamma^{2}_{\tau}+2C_{3}^{2}\sum_{\tau=1}^{t-1}\zeta^{2}_{\tau}\\
&+\norm{\lambda_{*}}_{2}^{2}\sum_{\tau=0}^{t-1}\tfrac{\eta_{\tau}}{2}
+ \sum_{\tau=0}^{t-1}\left( 2\eta_{\tau}^{2}-\frac{\eta_{\tau}}{2}+\frac{\beta_{\tau}^{2}\gamma_{\tau}^{2}C_{1}^{2}}{K}\right)\norm{\Lambda_{\tau}}^{2}_{2}
+S_{t}(x_{*})+\tfrac{2}{K}R_{t}
\end{split}
\end{equation}
To eliminate the first summand in above bound, we continue:
\begin{lemma}
	\label{Lem:aiaishsjdhhdjjsjsss}
It holds for all $\lambda\geq 0$:
\begin{equation*}
\inn{\Lambda_{\tau}-\lambda}{\mathbf{A}X_{\tau}-b}\leq \inn{X_{\tau}-\tilde{x}}{\mathbf{A}^{\T}\Lambda_{\tau}}-\inn{\lambda}{\mathbf{A}X_{\tau}-b},
\end{equation*}
where $\tilde{x}\in\mathcal{Q}$ arbitrary.
\end{lemma}
	Setting this observation into \eqref{Eq:aaiaiajssjshhdddd} and setting the choice $\zeta_{\tau}=\gamma_{\tau}$, it yields for any $\tilde{x}\in\mathcal{Q}$:
\begin{align}
\mathcal{V}_{t}^{(2)}(\lambda)&\leq \sum_{\tau=0}^{t-1}\gamma_{\tau}\inn{X_{\tau}-\tilde{x}}{\mathbf{A}^{\T}\Lambda_{\tau}}-\sum_{\tau=0}^{t-1}\gamma_{\tau}\inn{\lambda}{\mathbf{A}X_{\tau}-b} \nonumber\\ &~~~+\sum_{\tau=1}^{t-1}\left( 2\eta_{\tau}^{2}-\tfrac{\eta_{\tau}}{2}\right)\norm{\Lambda_{\tau}}^{2}_{2}+\norm{\lambda}_{2}^{2}\sum_{\tau=0}^{t-1} \tfrac{\eta_{\tau}}{2} +2C_{3}^{2}\sum_{\tau=0}^{t-1}\gamma_{\tau}^{2}
\nonumber
\end{align}

For $z_{*}=(x_{*},\lambda_{*})$ $\lambda_{*}\geq 0$, and for $\tilde{\lambda}\geq 0$, it holds by combining above inequality (with $\tilde{x}=x_{*}$) and \eqref{Eq:aaosjsjshdhfggfhhdggddhhsssaa}:
	\begin{align*}
	\mathcal{V}_{t}(z_{*})+\mathcal{V}_{t}^{(2)}(\tilde{\lambda})\leq
	&-\left( \inn{\tilde{\lambda}}{\sum_{\tau=0}^{t-1}\gamma_{\tau}\left( \mathbf{A}X_{\tau}-b\right) }-\frac{\eta t}{2}\norm{\tilde{\lambda}}^{2}_{2} \right) \nonumber\\
	&+\sum_{\tau=0}^{t-1}\gamma_{\tau} (2-\beta_{\tau})\inn{X_{\tau}-x}{\mathbf{A}^{\T}\Lambda_{\tau}}+\left( 4\eta_{\tau}^{2}-\eta_{\tau}+\frac{\beta_{\tau}^{2}\gamma_{\tau}^{2}C_{1}^{2}}{K}\right) \sum_{\tau=0}^{t-1}\norm{\Lambda_{\tau}}_{2}^{2}
	\\
	&+2\left(\tfrac{C_{2}^{2}}{K}+2C_{3}^{2}\right)\sum_{\tau=0}^{t-1}\gamma_{\tau}^{2}+\norm{\lambda_{*}}^{2}_{2}\sum_{\tau=0}^{t-1}\tfrac{\eta_{\tau} }{2}+ S_{t}(x_{*})+\tfrac{2}{K}R_{t}.
	\end{align*}
The second summand is eliminated by the choice $\beta_{\tau}=2$ and the third summand by the trackability condition. Thus it follows:
\begin{equation}
\label{Eq:aooajsjdhdhdhdggdffsggsss}
\begin{split}
\mathcal{V}_{t}(z_{*})+\mathcal{V}_{t}^{(2)}(\tilde{\lambda})\leq
&-\left( \inn{\tilde{\lambda}}{\sum_{\tau=0}^{t-1}\gamma_{\tau}\left( \mathbf{A}X_{\tau}-b\right) }-\frac{\eta t}{2}\norm{\tilde{\lambda}}^{2}_{2} \right) +2\left(\tfrac{C_{2}^{2}}{K}+2C_{3}^{2}\right)\sum_{\tau=0}^{t-1}\gamma_{\tau}^{2}\\
&+\norm{\lambda_{*}}^{2}_{2}\sum_{\tau=0}^{t-1}\tfrac{\eta_{\tau} }{2}+ S_{t}(x_{*})+\tfrac{2}{K}R_{t}.
\end{split}
\end{equation}

Now, since $\Lambda_{0}=0$, one sees that $\mathcal{V}_{t}^{(2)}(\tilde{\lambda})\geq -\norm{\tilde{\lambda}}^{2}_{2}/2$.
Moreover since $\Lambda_{0}=0$ and $Y_{0}=0$, we have $\mathcal{V}_{t}(z_{*})\geq -\Delta\psi(\X) -(\norm{\lambda_{*}}^{2}_{2}/2)$.
	Combining those observations with \eqref{Eq:aooajsjdhdhdhdggdffsggsss}, we obtain:
\begin{equation}
\label{Eq:aiaisjsjhdhdhddgdgddd}
\begin{split}
\left[ \inn{\tilde{\lambda}}{\sum_{\tau=0}^{t-1}\gamma_{\tau}(\mathbf{A}X_{\tau}-b)}-\tfrac{\sum_{\tau=0}^{t-1}\eta_{\tau} +1}{2}\norm{\tilde{\lambda}}_{2}^{2} \right] 
\leq &\Delta\psi(\X)
+2\left( \tfrac{C_{2}^{2}}{K}+C_{3}^{2}\right) \sum_{\tau=1}^{t-1}\gamma^{2}_{\tau}\\
&+\tfrac{(\sum_{\tau=0}^{t-1}\eta_{\tau} +1)}{2}\norm{\lambda_{*}}^{2}_{2}+ S_{t}(x_{*})+\tfrac{2}{K}R_{t}.
\end{split}
\end{equation}
Since:
\begin{equation*}
\sup_{\tilde{\lambda} \geq 0}\left(\inn{\tilde{\lambda}}{\sum_{\tau=0}^{t-1}\gamma_{\tau}\left( \mathbf{A}X_{\tau}-b\right) }-\frac{\sum_{\tau=0}^{t-1}\eta_{\tau}+1 }{2}\norm{\tilde{\lambda}}^{2}_{2} \right)=\tfrac{1}{2(\sum_{\tau=0}^{t-1}\eta_{\tau}+1)}\left\| \left[ \sum_{\tau=0}^{t-1}\gamma_{\tau}\left( \mathbf{A}X_{\tau}-b\right) \right]_{+}\right\|_{2}^{2}
\end{equation*}
Setting the optimizing $\tilde{\lambda}\geq 0$ into \eqref{Eq:aiaisjsjhdhdhddgdgddd}, taking the expectation of the resulted inequality, and noticing that $\Erw[S_{n}(x_{*})]=0$, since $S_{n}(x_{*})$ is a martingale with $\Erw[S_{1}(x_{*})]=0$, we obtain the desired statement with $\tilde{C}_{2}>0$ a constant satisfying $\norm{\lambda_{*}}_{2}^{2}$

\section{Proofs}
\begin{proof}[Proof of Lemma \ref{Lem:aaosshdjhdgdhhsgsgss}]
	Let be $\tau\in [t]_{0}$.
	By means of the bound for the Fenchel coupling of the one step difference given in \eqref{Eq:aaooshhsggdhhdgdgffdgdggdd}, we obtain:
	\begin{align*}
	F(x,Y_{\tau+1})&\leq F(x,Y_{\tau})+\gamma_{\tau}\inn{X_{\tau}-x}{v(X_{\tau})-\beta_{\tau}\mathbf{A}^{\T}\Lambda_{\tau}}+\gamma_{\tau}\inn{X_{\tau}-x}{\xi_{\tau+1}}\\
	&~~~+\frac{1}{2K}\norm{\gamma_{\tau} (v(X_{\tau})+\xi_{\tau+1})-\gamma_{\tau}\beta_{\tau}\mathbf{A}^{\T}\Lambda_{\tau}}_{*}^{2}.
	\end{align*}
	Triangle inequality and the inequality $(a+b)^{2}\leq 2 (a^{2}+b^{2})$ gives:
	\begin{align*}
	\norm{\gamma_{\tau} (v(X_{\tau})+\xi_{\tau+1})-\gamma_{\tau}\beta_{\tau}\mathbf{A}^{\T}\Lambda_{\tau}}_{*}^{2}&\leq
	2\gamma^{2}_{\tau}\norm{ (v(X_{\tau})+\xi_{\tau+1})}^{2}+2\gamma_{\tau}^{2}\beta_{\tau}^{2}\norm{\mathbf{A}^{\T}\Lambda_{\tau}}_{*}^{2}\\
	&\leq 4\gamma^{2}_{\tau}\norm{ v(X_{\tau})}^{2}+4\gamma^{2}_{\tau}\norm{\xi_{\tau+1}}^{2}+2\gamma_{\tau}^{2}\beta_{\tau}^{2}\norm{\mathbf{A}^{\T}\Lambda_{\tau}}_{*}^{2}\\
	&\leq 4\gamma^{2}_{\tau}C_{2}^{2}+4\gamma^{2}_{\tau}\norm{\xi_{\tau+1}}^{2}+2\gamma_{\tau}^{2}\beta_{\tau}^{2}C_{1}^{2}\norm{\Lambda_{\tau}}_{2}^{2}
	\end{align*}

	inserting the given iterate at time $\tau+1$, and applying triangle inequality, it holds
	\begin{align*}
	F(x,Y_{\tau+1})-F(x,Y_{\tau})&\leq\gamma_{\tau}\inn{X_{\tau}-x}{v(X_{\tau})-\beta_{\tau}\mathbf{A}^{\T}\Lambda_{\tau}}+\gamma_{\tau}\inn{X_{\tau}-x}{\xi_{\tau+1}}\\
	&~~~+\frac{\gamma_{\tau}^{2}}{K}\left(C_{1}^{2}\beta_{\tau}^{2}\norm{\Lambda_{\tau}}_{2}^{2}+2(C_{2}^{2}+\norm{\xi_{\tau+1}}_{*}^{2})\right).
	\end{align*}
	Summing above over all $\tau=0,\ldots,t-1$ and subsequent telescoping, we obtain an upper bound
\end{proof}
\begin{proof}[Proof of \ref{Lem:aaosshdggdhsgsghdgdgddd}]
	Let be $\tau\in [t]_{0}$	
	\begin{equation}
	\begin{split}
	&\norm{\Lambda_{\tau+1}-\lambda}^{2}_{2}=\norm{\Pi_{\real^{R}_{\geq 0}}((1-\eta_{\tau})\Lambda_{\tau}+\zeta_{\tau}[\mathbf{A}X_{\tau}-b])-\lambda}_{2}^{2}\\
	&\leq
	\norm{(1-\eta_{\tau})\Lambda_{\tau}+\zeta_{\tau}[\mathbf{A}X_{\tau}-b]-\lambda}_{2}^{2}=\norm{\Lambda_{\tau}-\lambda-\eta_{\tau}\Lambda_{\tau}+\zeta_{\tau}[\mathbf{A}X_{\tau}-b]}^{2}_{2}\\
	&=
	\norm{\Lambda_{\tau}-\lambda}^{2}_{2}
	+2\inn{\Lambda_{\tau}-\lambda}{-\eta_{\tau}\Lambda_{\tau}+\zeta_{\tau}[\mathbf{A}X_{\tau}-b]}+\norm{-\eta_{\tau}\Lambda_{\tau}+\zeta_{\tau}[\mathbf{A}X_{\tau}-b]}_{2}^{2}\\
	&=\norm{\Lambda_{\tau}-\lambda}^{2}_{2}
	-2\eta_{\tau}\inn{\Lambda_{\tau}-\lambda}{\Lambda_{\tau}}+\zeta_{\tau}\inn{\Lambda_{\tau}-\lambda}{\mathbf{A}X_{\tau}-b}+\norm{-\eta_{\tau}\Lambda_{\tau}+\zeta_{\tau}[\mathbf{A}X_{\tau}-b]}_{2}^{2}
	\end{split}
	\end{equation}
	Now, triangle inequality and the inequality $(a+b)^{2}\leq 2 (a^{2}+b^{2})$ gives:
	\begin{equation*}
	\norm{-\eta_{\tau}\Lambda_{\tau}+\zeta_{\tau}[\mathbf{A}X_{\tau}-b]}_{2}^{2}\leq 2\eta_{\tau}^{2}\norm{\Lambda_{\tau}}^{2}_{2}+2\zeta_{\tau}^{2}C_{3}^{2}.
	\end{equation*}
	Moreover, we have:
	\begin{equation*}
	\norm{\lambda}_{2}^{2}=\norm{\Lambda_{\tau}-\lambda-\Lambda_{\tau}}_{2}^{2}=\norm{\Lambda_{\tau}-\lambda}_{2}^{2}-2\inn{\Lambda_{\tau}-\lambda}{\Lambda_{\tau}}+\norm{\Lambda_{\tau}}_{2}^{2},
	\end{equation*}
	and consequently:
	\begin{equation*}
	-2\inn{\Lambda_{\tau}-\lambda}{\Lambda_{\tau}}= \norm{\lambda}_{2}^{2}-\norm{\Lambda_{\tau}}_{2}^{2}-\norm{\Lambda_{\tau}-\lambda}_{2}^{2}\leq\norm{\lambda}_{2}^{2}-\norm{\Lambda_{\tau}}_{2}^{2}
	\end{equation*}
	Combining all the observations, we obtain:
	\begin{align*}
	\norm{\Lambda_{\tau+1}-\lambda}^{2}_{2}-\norm{\Lambda_{\tau+1}-\lambda}^{2}_{2}&\leq\zeta_{\tau}\inn{\Lambda_{\tau}-\lambda}{\mathbf{A}X_{\tau}-b}+\left( 2\eta_{\tau}^{2}-\tfrac{\eta_{\tau}}{2}\right)\norm{\Lambda_{\tau}}^{2}_{2}+ \tfrac{\eta_{\tau}}{2}\norm{\lambda}_{2}^{2} +2\zeta_{\tau}^{2}C_{3}^{2}.
	\end{align*}
	Summing above over all $\tau=0,\ldots,t-1$ and subsequent telescoping, we obtain the desired upper bound.
\end{proof}
\begin{proof}[Proof of Lemma \ref{Lem:aiaishsjdhhdjjsjsss}]
	we have for any $x\in \mathcal{X}$:
	\begin{align}
	\inn{\Lambda_{\tau}-\lambda}{\mathbf{A}X_{\tau}-b}&=\inn{\Lambda_{\tau}}{\mathbf{A}X_{\tau}-b}-\inn{\lambda}{\mathbf{A}X_{\tau}-b}\nonumber\\
	&=\inn{\Lambda_{\tau}}{\mathbf{A}X_{\tau}-\mathbf{A}x}+\inn{\Lambda_{\tau}}{\mathbf{A}x-b}-\inn{\lambda}{\mathbf{A}X_{\tau}-b}\nonumber\\
	&=
	\inn{X_{\tau}-x}{\mathbf{A}^{\T}\Lambda_{\tau}}+ \inn{\Lambda_{\tau}}{\mathbf{A}x-b}-\inn{\lambda}{\mathbf{A}X_{\tau}-b}.\label{Eq:aaiishdggdhhsgsgsdd}
	\end{align}
	Now, since $\Lambda_{\tau}\geq 0$, it follows that $\inn{\Lambda_{\tau}}{\mathbf{A}\tilde{x}-b}\leq 0$, for $\tilde{x}\in\mathcal{Q}$. Therefore, if we set $x=\tilde{x}$ with $x_{*}\in\mathcal{Q}$ in \eqref{Eq:aaiishdggdhhsgsgsdd}, we have from previous observation the desired statement.	
\end{proof}
\section{Mirror Map and Fenchel Coupling}
\begin{proposition}
	\label{Prop:aaisshhfjffjfjfff}
	Let $F$ be the Fenchel coupling induced by a $K$-strongly convex regularizer on a compact convex subset $\mathcal{Z}$ of a Euclidean normed space $\V$. For $p\in\mathcal{
		Z}$, $y,y^{'}\in \V^{*}$, we have:
\begin{align}
F(p,y)&\geq (K/2)\norm{\Phi(y)-p}^{2}\\
F(p,y^{'})&\leq F(p,y)+\inn{\Phi(y)-p}{y^{'}-y}+(1/2K)\norm{y^{'}-y}^{2}_{*}
\label{Eq:aaooshhsggdhhdgdgffdgdggdd}
\end{align}
\end{proposition}
For a proof, see e.g.  Proposition 4.3 (c) in \cite{Mertikopoulos2018}
\section{Decoupling the Coupled Constraints by means of Lagrangian Method}
In order to investigate $\VI(\mathcal{Q},v)$, it is convenient to extend $\VI(\mathcal{Q},v)$ to $\VI(\mathcal{X}\times \real^{R}_{\geq 0},\tilde{v})$, where $\tilde{v}:\mathcal{X}\times \real^{R}_{\geq 0}$,
\begin{equation}
\label{Eq:iaaiisjsjdhhdhdhdhdjjsss}
\tilde{v}:\mathcal{X}\times \real^{R}_{\geq 0}\rightarrow \real^{D+R},\quad (x,\lambda)\mapsto \left[\begin{array}{c}
v(x)-\mathbf{A}^{\T}\lambda\\
\mathbf{A}x- b
\end{array}
\right].
\end{equation}
The advantage of this method is the decoupling of the constraint set. Specifically, by means of this procedure, we only have to work with the constraint set $\X\times \real^{M}_{\geq 0}$ with product structure rather than with $\mathcal{Q}$. 

Usual KKT argumentation asserts that $\VI(\mathcal{Q},v)$ and $\VI(\mathcal{X}\times \real^{M}_{+},\tilde{v})$ is essentially the same in the following sense:
\begin{proposition}
	\label{Prop:aiaisshshjdhdggddd}
	Suppose that Assumption \ref{Eq:aioaoosjjshhdjjddss} holds. Then the following statements are equivalent:
	\begin{enumerate}
		\item $\overline{x}\in\mathcal{Q}$ is a solution of $\VI(\mathcal{Q},v)$
		\item There exists $\overline{\lambda}\in\real^{R}_{\geq 0}$ s.t. $(\overline{x},\overline{\lambda})$ is a solution of $\VI(\mathcal{X}\times \real^{R}_{\geq 0},\tilde{v})$. 
	\end{enumerate}
\end{proposition}
A proof of this statement is provided in Appendix \ref{App:aoaosjjshshhdhdshsjsss}.
Proposition \ref{Prop:aiaisshshjdhdggddd} tell us that in order to find a variational Nash equilibrium of $\Gamma$ (in case it exists), it is sufficient to find the solution of $\VI(\mathcal{X}\times \real^{R}_{\geq 0},\tilde{v})$. 
\section{KKT-System for Variational Inequality}
\label{App:aoaosjjshshhdhdshsjsss}
Let be $M,D\in\nat$. We consider the problem:
\begin{equation}
\label{Eq:aaoossjjdhdhdjshshs}
\max_{x\in \mathcal{D}} f(x)\quad\text{s.t. }g(x)\leq 0,
\end{equation}
where $f:\mathcal{D}\rightarrow\real$, and $g:\mathcal{D}\rightarrow\real^{M}$ continuously differentiable and $\mathcal{D}$ is a non-empty subset of $\real^{D}$. We define $\mathcal{Z}:=\lrbrace{x\in \mathcal{D}:~g(x)\leq 0}$.
The Karush-Kuhn-Tucker (KKT) system corresponding to \eqref{Eq:aaoossjjdhdhdjshshs} is defined as:
\begin{align}
&\nabla f(x)-\nabla g(x)^{T}\lambda =0\nonumber\\
&0\leq \lambda~\bot~ g(x)\leq 0\nonumber\\
&x\in\mathcal{D} \label{Eq:aaoajsjskshhddjdhdhdd}
\end{align} 

\begin{definition}[Slater constraint qualification]
	The set $\mathcal{Z}$ is said to satisfy Slater's constraint qualification if $g_{i}$ is convex for all $i\in [M]$ and if there exists $x_{*}\in\text{relint}(D)$ s.t. $g(x_{*})<0$.
\end{definition}
The following statement gives a relation between a convex optimization problem and its KKT system:
\begin{proposition}
	\label{Prop:aoaoajsshdjdhhsjsjss}
	Suppose that $\mathcal{Z}$ satisfies the Slater's CQ. Then:
	\begin{enumerate}
		\item If $\overline{x}$ solves \eqref{Eq:aaoossjjdhdhdjshshs}, then there exists $\overline{\lambda}$ s.t. $(\overline{x},\overline{\lambda})$ solves \eqref{Eq:aaoajsjskshhddjdhdhdd} 
		\item Assume that $f$ is convex on $\mathcal{D}$. If $(\overline{x},\overline{\lambda})$ solves \eqref{Eq:aaoajsjskshhddjdhdhdd} then $\overline{x}$ solves \eqref{Eq:aaoossjjdhdhdjshshs}
	\end{enumerate}
\end{proposition}

The idea of relating constrained convex optimization problem to unconstrained convex optimization problem applies also to variational inequality. For this sake, we define the notion of KKT system relative to the variational inequality $\VI(\mathcal{Z},F)$ 
by:
\begin{align}
&F(x)-\lambda^{T}\nabla g(x)=0\nonumber\\
&0\leq \lambda~\bot ~g(x)\leq 0\nonumber\\
&x\in\mathcal{Z}.\label{Eq:aiaisshjhdhdhddd}
\end{align}

\begin{proposition}
	\label{Prop:aossjdhdhhdjjssss}
	Suppose that $\mathcal{Z}$ satisfies the Slater's CQ. Then $\overline{x}$ solves $\VI(\mathcal{Z},F)$ if and only if there exists $\overline{\lambda}$ s.t. $(\overline{x},\overline{\lambda})$ solves the KKT system \eqref{Eq:aiaisshjhdhdhddd}
\end{proposition}
\begin{proof}
	$\overline{x}$ solves $\VI(\mathcal{Z},F)$ if and only if:
	\begin{equation*}
	\inn{x}{F(\overline{x})}\leq\inn{\overline{x}}{F(\overline{x})},\quad\forall x\in\mathcal{Z}.
	\end{equation*}
	In turn, the latter is equivalent with:
	\begin{equation*}
	\overline{x}\in\argmax_{x\in\mathcal{Z}}\inn{x}{F(\overline{x})}.
	\end{equation*}
	The KKT system corresponds to above program is given by \eqref{Eq:aiaisshjhdhdhddd} (cf. \eqref{Eq:aaoossjjdhdhdjshshs})
	Since $x\mapsto\inn{x}{F(\overline{x})}$ is trivially convex on $\mathcal{D}$, and $\mathcal{Z}$ satisfies the Slater's constraint qualification, then the desired statement follows from Proposition \ref{Prop:aoaoajsshdjdhhsjsjss}. 
\end{proof}
In order to obtain Proposition \ref{Prop:aiaisshshjdhdggddd} all we need is to get rid of the condition $\lambda\bot g(x)$ in the KKT system for $\VI$ (see \eqref{Eq:aiaisshjhdhdhddd}). This is done in the following:
\begin{proof}[Proof of Proposition \ref{Prop:aiaisshshjdhdggddd}]
	Since there are no explicit inequality constraint in $\mathcal{X}\times\real^{R}_{\geq 0}$ then $\mathcal{X}\times\real^{R}_{\geq 0}$ fulfills the Slater's CQ. It follows from Proposition \ref{Prop:aossjdhdhhdjjssss} that $\SOL(\mathcal{X}\times\real^{R}_{\geq 0},\tilde{v})$ coincides with the solution of the KKT system:
	\begin{align}
	&v(x)-\nabla \phi(x)\lambda=0\nonumber\\
	&\phi(x)+\mu=0\label{Eq:aiaisshjhdhdhdddfffksdks1}\\
	&0\leq \mu~\bot ~-\lambda\leq 0\label{Eq:aiaisshjhdhdhdddfffksdks2}\\
	&x\in\X~~\lambda\in\real^{R}_{\geq 0}\label{Eq:aiaisshjhdhdhdddfffksdks}.
	\end{align}
	Setting \eqref{Eq:aiaisshjhdhdhdddfffksdks1} into \eqref{Eq:aiaisshjhdhdhdddfffksdks2} we obtain that  $\SOL(\mathcal{X}\times\real^{R}_{\geq 0},\tilde{v})$ coincides with the solution of KKT system: 
	\begin{align}
	&v(x)-\nabla \phi(x)\lambda=0\nonumber\\
	&0\leq \lambda~\bot ~\phi(x)\leq 0\nonumber\\
	&x\in\X\label{Eq:aiaisshjhdhdhdddfffksdksddd}.
	\end{align}
	For the final step, notice that the solution of this KKT system is equivalent to $\SOL(\mathcal{Q},v)$ since the Slater's CQ for $\mathcal{Q}$ is fulfilled.
\end{proof}
\end{document}